\documentclass[runningheads]{llncs}

\usepackage{eccv}

\usepackage{eccvabbrv}

\usepackage{graphicx}
\usepackage{booktabs}
\usepackage[accsupp]{axessibility}  

\usepackage{xspace}
\usepackage{amsmath}
\usepackage{amssymb}
\usepackage{mathtools}
\usepackage[table]{xcolor}
\usepackage{algorithm}
\usepackage{algorithmic}
\usepackage{adjustbox}
\usepackage{tabularx}
\usepackage[table]{xcolor}
\usepackage{subcaption}
\definecolor{gray}{HTML}{efefef}
\newcolumntype{x}[1]{>{\centering\arraybackslash}p{#1}}
\newcolumntype{y}[1]{>{\raggedright\arraybackslash}p{#1}}
\newcolumntype{z}[1]{>{\raggedleft\arraybackslash}p{#1}}
\newcommand{\tablestyle}[2]{\setlength{\tabcolsep}{#1}\renewcommand{\arraystretch}{#2}\centering\footnotesize}
\usepackage{pifont}
\newcommand{\CheckmarkBold}{\ding{51}}  
\newcommand{\XSolidBrush}{\ding{55}}    
\newcommand{\dui}{\CheckmarkBold}
\newcommand{\budui}{\XSolidBrush}

\usepackage[]{hyperref}

\usepackage{orcidlink}

\begin{document}
\newcommand{\modelname}{\texttt{WorldBagel}\xspace}

\title{\modelname: Uncovering the Power of Unified Multimodal Models for Vision-Language-Action-World Modeling}

\titlerunning{\modelname}


\author{Zelin Zhao, Min Shi, Bo Yuan, Haotian Xue, Jialuo Li, Lama Moukheiber, Humphrey Shi, and Yongxin Chen}

\authorrunning{Z.Zhao et al.}

\institute{Georgia Institute of Technology}

\maketitle

\begin{abstract}
World models aim to capture environment dynamics to support perception, reasoning, and action, and have recently risen as a prominent direction in \textbf{V}ision–\textbf{L}anguage–\textbf{A}ction–\textbf{W}orld (VLAW) modeling. While recent unified vision–language models demonstrate unprecedented multimodal generation ability, we argue that their role as world models remains largely uncovered. In this work, we introduce \texttt{WorldBagel}, a unified VLAW framework built on BAGEL, a modern multimodal unified model, and use it to systematically study the impact of unification in VLAW modeling. Through multi-task robotic manipulation and cross-domain experiments, we show that unified models consistently outperform task-specific alternatives, learn higher-quality action representations that are more structured and semantically aligned with visual and linguistic context. Our results on LIBERO, Language Table, and Franka suggest that unification itself benefits learning effective VLAW models, yielding both consistent empirical gains and deeper insights into multimodal world modeling. The code and model checkpoints will become available after acceptance.
\end{abstract}

\section{Introduction}
\label{sec:intro}
World models aim to capture environment dynamics to support perception, reasoning, and decision-making, and have become a central objective in embodied intelligence and robotics research~\cite{ha2018world,hafner2019learning,schmidhuber2015world}. Recent approaches learn latent dynamics models that enable planning, imagination, or policy learning directly from visual observations~\cite{hafner2020dreamer}. These works demonstrate that modeling future states can substantially improve control and sample efficiency. However, most world modeling efforts focus primarily on visual dynamics and control, without deeply integrating language or large-scale multimodal knowledge.

In parallel, large-scale vision–language models (VLMs) have achieved remarkable progress in multimodal representation learning and generation~\cite{alayrac2022flamingo,li2023blip2}. These models align visual and textual modalities at scale, enabling strong zero-shot transfer and reasoning capabilities. Building upon this line of work, vision–language–action (VLA) models extend multimodal learning to embodied control, showing that incorporating language improves multi-task learning and generalization in robotics~\cite{brohan2022rt1,intelligence2025pi05visionlanguageactionmodelopenworld}. Despite their success, VLA systems typically emphasize action prediction and task performance, rather than explicit generative modeling of environment dynamics.

Taken together, these two research directions—multimodal VLA learning and world modeling—suggest a natural convergence. A unified model should not only understand visual observations and follow language instructions, but also generate future visual states and predict actions within a shared representation space. This motivates the \textbf{V}ision–\textbf{L}anguage–\textbf{A}ction–\textbf{W}orld (VLAW) framework, where a single model jointly models perception, language grounding, action prediction, and future observation generation. In this setting, the model must align visual dynamics with linguistic goals and embodied actions, while maintaining the generative capability to predict future frames or latent states.

A growing body of work suggests that unification across modalities and tasks is a key ingredient for scalable world models. Unified architectures have been shown to improve transfer, robustness, and sample efficiency in language and vision domains~\cite{radford2021clip,jaegle2021perceiver,reed2022generalist}. In robotics, recent evidence indicates that training a single model across diverse tasks and embodiments can lead to emergent generalization capabilities~\cite{brohan2023rt2,intelligence2025pi05visionlanguageactionmodelopenworld}. More recently, unified multimodal models such as BAGEL~\cite{bagel2024} adopt a two‑tower architecture composed of dedicated generation (GEN) and understanding (UND) experts, enabling both high‑quality multimodal generation and strong discriminative reasoning within a unified framework. By combining modality‑specific encoders with specialized generative and understanding experts, BAGEL establishes new state‑of‑the‑arts across diverse multimodal benchmarks. Nevertheless, it remains unclear how unified multimodal models should be adapted or evaluated as general VLAW models.

In this work, we introduce~\modelname, a unified VLAW framework that extends BAGEL’s two‑tower GEN/UND architecture to jointly support multimodal understanding, action prediction, and future frame generation within a unified model. To situate our approach within the broad landscape of VLAW methods, we provide a structured methodological comparison in~\Cref{tab-vlaw-comparison}. The \textbf{technical contributions} of~\modelname include:
\begin{enumerate}
    \item \textbf{Fourier feature action tokenizer and decoder.}
    We encode continuous control signals using Fourier feature embeddings aligned with the shared multimodal token space, improving action representation quality.

    \item \textbf{Interleaved VLAW modeling via sequence plans.} We introduce a multi-step sequence-plan framework that organizes visual observations, language instructions, action tokens, and future state predictions within a unified autoregressive sequence. During training, sequence plans are sampled from a set of defined orderings to jointly conduct multimodal learning.

    \item \textbf{LLM-Inspired VLAW data sampling.}
    We adopt an adaptive dataset reweighting strategy originated from multimodal pre-training to stabilize heterogeneous multimodal training and improve cross-task robustness.
\end{enumerate}

Through extensive experiments on robotic manipulation benchmarks, including LIBERO~\cite{li2023libero}, Language Table~\cite{lynch2019learning}, and Franka~\cite{isaacsim2024}, we systematically evaluate unified VLAW models. Our study reveals three \textbf{empirical findings}:
\begin{enumerate}
    \item \textbf{Multi-task performance.}
    Unified VLAW models achieve strong multi-task manipulation performance, outperforming various strong VLA baselines.

    \item \textbf{Action representation quality.}
    Fourier action modules improve action prediction accuracy and produce more informative action embeddings.

    \item \textbf{Stability under distribution shifts.}
    Under distribution shifts in action dynamics, unified models maintain more stable world predictions and exhibit richer representation structure as evidenced by eigenvalue spectrum analysis.
\end{enumerate}
\begin{figure}[t]
    \centering
    \includegraphics[width=\linewidth]{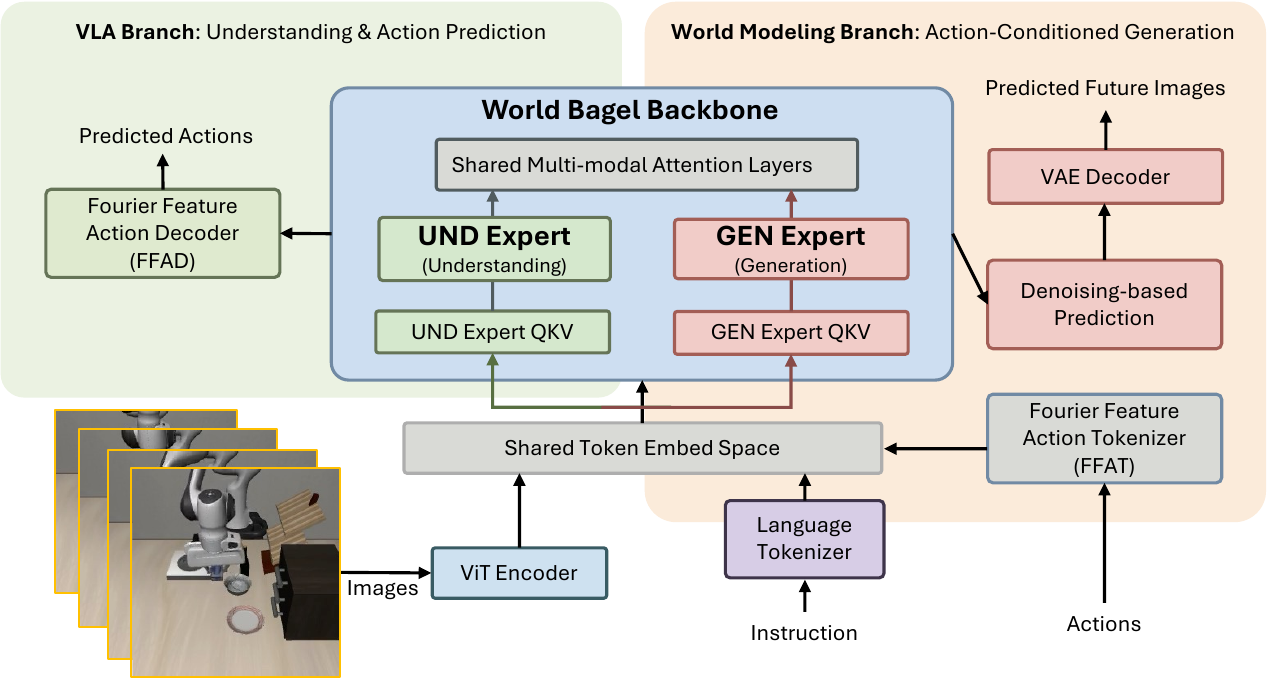}
    \caption{
    Overview of \modelname. The model is built on a shared multimodal backbone with two experts: an Understanding (UND) expert and a Generation (GEN) expert, operating over a unified token embedding space. 
    Visual observations and language instructions are encoded by a ViT encoder and a language tokenizer, while continuous control actions are represented using a Fourier Feature Action Tokenizer (FFAT). \modelname consists of two branches: the VLA branch and the World Modeling branch.
    The VLA branch leverages the UND expert to perform vision–language grounding and predict robot actions through a Fourier Feature Action Decoder (FFAD). 
    In parallel, the World Modeling branch conditions on past visual tokens and action tokens produced by FFAT, and uses the GEN expert to generate future observations via denoising-based prediction followed by a VAE decoder. 
    Both branches share multimodal attention layers but employ expert-specific QKV projections and task-specific decoders.
    }
    \label{fig:model}
\end{figure}
\section{Related Works}
\label{sec-related-works}
\begin{table*}[t]
\begin{center}
\resizebox{\textwidth}{!}{
\begin{tabular}{
x{75pt}|
x{60pt}
x{75pt}
x{60pt}
x{50pt}
x{50pt}
x{70pt}
x{60pt}}
\toprule 
Method &
Publication Venue &
Base Model &
Training Method &
Language Grounding &
Action-Conditioned &
World Modeling Predictions&
Action Tokenizer \\
\midrule

Dreamer~\cite{hafner2020dreamer} 
& ICLR'20 
& RSSM 
& From scratch 
& \budui 
& \dui 
& Future states 
& \budui \\

DINO-WM~\cite{dinowm}
& ICML'25 
& DINOv2 
& CEM
& \budui 
& \dui 
& Future states 
& \budui \\

RT-2~\cite{brohan2023rt2} 
& CoRL'23 
& PaLI-X / PaLM-E
& SFT
& \dui
& \budui 
& N/A 
& \budui \\

OpenVLA~\cite{kim2024openvla} 
& CoRL'24
& Llama 2
& SFT
& \dui 
& \budui
& N/A
& \budui \\ 

$\pi_{0.5}$~\cite{intelligence2025pi05visionlanguageactionmodelopenworld}
& arXiv'25 
& PaliGemma 
& Two-phase
& \dui 
& \budui 
& N/A
& FAST \\ 
\midrule

RynnVLA-002~\cite{cen2025rynnvla002}
& arXiv'25 
& Chameleon 
& SFT
& \dui 
& \dui 
& Future frames 
& Bin discretize \\

BagelVLA$^\dagger$~\cite{hu2026bagelvla} 
& arXiv'26 
& BAGEL
& SFT 
& \dui 
& \budui 
& Future frames  
& \budui \\

\rowcolor{cyan!8} \modelname 
& N/A
& BAGEL 
& SFT 
& \dui 
& \dui
& Future frames 
& Fourier \\
\bottomrule
\end{tabular}}
\end{center}

\caption{\textbf{Methodology comparison of recent most related works.} 
We compare \modelname against representative world models and VLA systems across key dimensions of VLAW modeling.
Methods marked with $^\dagger$ denote concurrent unpublished work developed independently and in parallel with this submission. See~\Cref{sec-related-works} for details.}
\label{tab-vlaw-comparison}
\end{table*}

\paragraph{World Models.}
World models aim to learn predictive representations of environment dynamics for planning, control, and imagination~\cite{ha2018world,schmidhuber2015world,hafner2019learning,hafner2020dreamer,zhao2022tracking,xue2026acwmphysinvestigatinggeneralizedphysical,proto,zhao2021augmenting,zhao2025physicsinformedneuraloperatorpredictivecontrol}. Beyond classical latent dynamics models, recent work has expanded toward large-scale video and generative world modeling. Video diffusion and auto-regressive approaches~\cite{xie2024show,songhistory,wan2025wan,CETCam,liu2026byteloomweavinggeometryconsistenthumanobject,song2026MVS2V} explore scalable video generation and temporal consistency. Meanwhile, world-centric simulators and embodied benchmarks—
including \\WorldGym~\cite{worldgym}, 
DREAM4~\cite{dream4}, DreamDojo~\cite{gao2026dreamdojo}, Ctrl-world~\cite{ctrlworld}, and Semantic World~\cite{berg2025semantic}—emphasize predictive modeling for interactive and physically grounded environments. Other efforts such as DINO-WM~\cite{dinowm}, V-JEPA~\cite{vjepa2}, and Hunyuan World~\cite{team2025hunyuanworld} further investigate representation-driven or generative world dynamics. Despite this diversity, most existing world models either focus on visual prediction alone or treat action and language as auxiliary signals rather than jointly optimized generative components. Concurrent works Rynnvla-002~\cite{cen2025rynnvla002} and BagelVLA~\cite{hu2026bagelvla} considers similar tasks, while they have not systematically study many VLAW design choices, such as action tokenizer or sampling strategy.

\paragraph{Vision–Language–Action (VLA) Models.}
VLA models extend multimodal learning to embodied control by conditioning policies on language instructions~\cite{brohan2022rt1,brohan2023rt2,intelligence2025pi05visionlanguageactionmodelopenworld}. Recent systems such as OpenVLA~\cite{kim2024openvla}, UniVLA~\cite{bu2025univla}, InternVLA~\cite{chen2025internvla}, Octo~\cite{team2024octo}, Cosmos Policy~\cite{cosmosPolicy}, and $\pi_{0.5}$~\cite{intelligence2025pi05visionlanguageactionmodelopenworld} demonstrate scaling trends in unified perception-action modeling. Parallel lines of work investigate reinforcement-driven alignment of multimodal models, including RL4VLM~\cite{RL4VLM} and Unified-GRPO~\cite{unifiedGRPO}. These efforts highlight the importance of reasoning, alignment, and interactive learning in multimodal agents. However, most established VLA systems emphasize policy optimization or decision-making, while explicit generative modeling of future world states remains secondary or loosely integrated.

\paragraph{Unified Generation Models.} Unified generation models~\cite{vilaU,xie2025muse,janus,vilaU,lapflow} integrate the capability to generate content in multiple modalities (e.g., text, images and videos) into a single architecture, enabling visual chain-of-thought~\cite{zebracot} and interactive world modeling \cite{emu3_5, bagel2024}. Early works~\cite{chameleon,emu} quantize images into discrete tokens to unify generation via next-token prediction, whereas Transfusion~\cite{transfusion} hybridizes diffusion for images with next-token prediction for text. Furthermore, as single visual representation often struggles to balance reconstruction for generation with semantic abstraction for understanding, recent works~\cite{janus} decouple these pathways. BAGEL~\cite{bagel2024} further decouples generation and understanding parameters through a Mixture-of-Transformer architecture, achieving performance comparable to state-of-the-art single-modality generation models. Nevertheless, how unified models can be adapted to jointly model vision, language, action, and future world dynamics remains insufficiently studied~\cite{hu2026bagelvla}. Our work positions unified VLAW modeling as this missing convergence point.

\vspace{-2mm}
\section{Methodology}
\label{sec:method}

We present \modelname, a \textbf{V}ision–\textbf{L}anguage–\textbf{A}ction–\textbf{W}orld (VLAW) model built upon BAGEL’s two‑tower GEN/UND architecture. 
Our objective is to extend a large-scale unified multimodal generative model beyond perception and reasoning, enabling it to jointly support (1) multimodal understanding, (2) structured action modeling, and (3) future world prediction within a single VLAW framework to unify perception, decision-making, and environment modeling.

\subsection{VLAW Formulation and Model Framework}

\paragraph{VLAW Formulation.}
At time step $t$, the agent observes an image $o_t$, receives a language instruction $l$, executes a continuous action $a_t \in \mathbb{R}^d$, and transitions to the next observation $o_{t+1}$. 
Given a trajectory $\tau = \{(o_1, a_1), \dots, (o_T, a_T)\}$, we model the joint distribution
\(
p(o_{t+1}, a_t \mid o_{\le t}, a_{<t}, l).
\)
Unlike classical world models~\cite{dream4,ha2018world} that learn latent dynamics separately from policy learning, the VLAW formulation jointly models action prediction and future observation generation within a unified multimodal architecture.

\paragraph{Two-Tower Backbone.}
As shown in Fig.~\ref{fig:model}, \modelname inherits BAGEL’s two‑tower design consisting of UND expert for multimodal understanding, and GEN expert for generative modeling.
All modalities (vision, language, action) are embedded into a shared token embedding space and processed by a common multimodal self-attention backbone. 
Within each transformer block, query–key–value projections are routed through either the UND or GEN expert depending on the training objective. Importantly, we do not introduce an additional action expert; instead, action modeling is achieved through fine-tuned action tokenizer and decoder, following unified token-based modeling~\cite{reed2022generalist}.

\paragraph{VLA Branch: Understanding and Action Prediction.} As shown in the left of~\Cref{fig:model}, for action prediction, the input sequence contains visual state tokens $V_t$ and language tokens $L$ (describing the goal). In the VLA Branch, visual states are encoded with a ViT encoder. These tokens are processed through the UND expert, which specializes in cross-modal grounding. An action decoder is attached on top of the UND pathway to predict the continuous control signal $a_t$ from special action token embeddings. Refer to~\Cref{sec:ffad} for more details.

\paragraph{World Modeling Branch: Action-Conditioned Generation.} In the world modeling branch, visual observations are encoded into latent representations using a VAE encoder, while predicted future states are reconstructed using the corresponding VAE decoder. The world modeling branch conditions on previous visual tokens $V_t$, language tokens $L$, and structured action tokens $A_t$. 
Actions are first transformed using our Fourier Feature Action Tokenizer (FFAT, detailed in~\Cref{sec:ffat}) and embedded into the shared token space. The resulting sequence is processed through the GEN expert, which further predicts the next-step observation $V_{t+1}$ via flow matching~\cite{bagel2024}.

\paragraph{Training Parameters and Objective.}
We initialize from a pretrained BAGEL checkpoint and perform supervised fine-tuning (SFT) on VLAW trajectories. During SFT, we update all UND and GEN expert parameters, the action prediction head, the Fourier action tokenizer projections, and the generative vision decoder, while keeping the vision and language tokenizers frozen. The overall training loss $\mathcal{L}$ combines action prediction and future frame generation:
\begin{equation}
\mathcal{L} =
\mathcal{L}_{action}
+
0.1\mathcal{L}_{vision}.
\end{equation}
Here, $\mathcal{L}_{action}$ is a regression loss over actions in UND branch, and $\mathcal{L}_{vision}$ is the flow matching loss for future visual tokens in the GEN branch~\cite{bagel2024}.

\subsection{Fourier Feature Action Decoder (FFAD)}
\label{sec:ffad}

In the VLA branch of~\modelname, the model predicts continuous control signals conditioned on visual and language tokens. Standard regression treats actions as independent real-valued outputs, while discretization-based tokenizers (e.g., uniform binning or FAST tokenization~\cite{intelligence2025pi05visionlanguageactionmodelopenworld,pertsch2025fast}) convert continuous controls into categorical tokens. Furthermore, FAST~\cite{pertsch2025fast} tokenization relies on a BPE training procedure~\cite{bpe} to learn action sub-tokens, making the tokenization data-dependent and potentially unstable across domains or control ranges.
We instead propose a \textbf{Fourier Feature Action Decoder (FFAD)}, leveraging the bi-directional mapping between actions and structured \textit{continuous} tokens.

\paragraph{Forward Tokenization (Action $\rightarrow$ Fourier Features).}
Given a continuous action vector $a_t \in \mathbb{R}^d$, we map each dimension into a Fourier feature~\cite{FourierFeaturesLearn,zhao2024grounding,nerf}:
\begin{equation}
\phi(a_t) =
\left[
\sin(2^0 \pi a_t), \cos(2^0 \pi a_t), \dots,
\sin(2^K \pi a_t), \cos(2^K \pi a_t)
\right].
\label{eq-fourier-feature}
\end{equation}
This expands each action component into frequency bands with exponentially increasing scales, capturing both coarse and fine-grained variations. 
The resulting Fourier features are linearly projected into the shared multimodal token space and treated as structured action tokens in the autoregressive sequence.
\paragraph{Prediction in Fourier Space.}
Instead of directly regressing $a_t$, the model predicts $\hat{\phi}(a_t)$ in Fourier space. 
The action loss is computed as an $\ell_2$ regression objective:
\begin{equation}
\mathcal{L}_{action} = 
\|\hat{\phi}(a_t) - \phi(a_t)\|_2^2.
\end{equation}
Prior findings~\cite{zhao2024grounding,FourierFeaturesLearn,jiang2018pointsiftsiftlikenetworkmodule} indicate predicting Fourier features improves stability by distributing information across frequency components and avoiding direct regression in the raw space, which promote learning high-frequency signals.
\paragraph{Inverse Mapping (Fourier Features $\rightarrow$ Action).} To deterministically reconstruct actions from predicted Fourier features, we aggregate phase estimates across all frequencies with phase-consistent averaging. 
First, for each frequency $k$, we normalize the predicted sine–cosine pair $(\hat{s}_k, \hat{c}_k)$ onto the unit circle:
\begin{equation}
(\tilde{s}_k, \tilde{c}_k) = 
\frac{(\hat{s}_k, \hat{c}_k)}
{\sqrt{\hat{s}_k^2 + \hat{c}_k^2 + \epsilon}},
\end{equation}
where $\epsilon = 10^{-8}$ ensures numerical stability. We then recover the phase angle $\theta_k = \operatorname{atan2}(\tilde{s}_k, \tilde{c}_k)$ and scale it back to action space via $\tilde{a}_k = \theta_k / (2^k \pi)$. 
To ensure consistency across periodic branches, we use the lowest-frequency estimate ($k=0$) as an anchor and unwrap higher-frequency phases by adding integer multiples of $2\pi$ to minimize deviation from $\tilde{a}_0$. The final action is obtained by equal-weight averaging across frequencies:
\begin{equation}
\hat{a}
= \frac{1}{K+1} \sum_{k=0}^{K} \tilde{a}_k=
\frac{1}{K+1}
\sum_{k=0}^{K}
\frac{
\operatorname{atan2}(\tilde{s}_k, \tilde{c}_k)
}{
2^k \pi
},
\end{equation}
yielding a reconstruction that preserves continuity and local metric structure of the original action space without requiring iterative optimization. Since both the Fourier embedding and phase-based inversion are continuous mappings~\cite{FourierFeaturesLearn,zhao2024grounding}, the overall reconstruction is locally Lipschitz, ensuring small perturbations in predicted features induce small changes in reconstructed actions. We provide more mathematical analysis of FFAD in the appendix~\Cref{app-math-behind-fatt-ffad}.

\paragraph{Architecture Details.}
FFAD is implemented as a lightweight adapter (33M parameters) on top of the BAGEL backbone. After visual and language tokens are processed by the UND expert, we append one or multiple special \texttt{<|action\_pred|>} tokens to the autoregressive sequence, where each appended token corresponds to one action step. The hidden state at each action anchor position $h_a^{(k)} \in \mathbb{R}^{H}$ serves as the query representation for the $k$-th action. For single-step control, a single action token is appended; for multi-step prediction with horizon $H$, $H$ consecutive action tokens are appended, enabling autoregressive multi-action generation within the same forward pass. Each hidden state is first projected to a reduced dimension $d_r=1024$ via a linear layer, then refined by a 2-layer Transformer encoder (4 attention heads, feedforward dimension $4d_r$, LayerNorm and dropout), and finally passed through a 2-block residual MLP head to produce the action output in $\mathbb{R}^{d}$. This design isolates action-specific learning within a compact head while preserving the multimodal reasoning capacity of BAGEL.

\subsection{Fourier Feature Action Tokenizer (FFAT)}
\label{sec:ffat}

In the world modeling branch, actions are used as conditioning signals for future frame generation rather than prediction targets. We therefore introduce the \textbf{Fourier Feature Action Tokenizer (FFAT)}, which converts continuous actions into structured multimodal tokens compatible with the GEN expert. Given a continuous action $a_t \in \mathbb{R}^d$, we compute the same multi-frequency Fourier representation as~\Cref{eq-fourier-feature}. During training, ground-truth actions are encoded with FFAT to condition future frame prediction. During inference, actions generated by FFAD are first reconstructed into continuous space and then re-encoded via FFAT before being fed into the world modeling branch. This shared Fourier formulation ensures consistency between action prediction and action-conditioned generation while avoiding discretization artifacts~\cite{pertsch2025fast}.

The Fourier features are directly passed through a linear projection layer:
\begin{equation}
    z_a = W_f \, \phi(a_t),
\quad W_f \in \mathbb{R}^{D \times \dim(\phi)},
\end{equation}
to map them into the shared multimodal token space of dimension $D$. The resulting action tokens are appended to the autoregressive sequence and processed by the GEN expert together with visual and language tokens. For multi-step conditioning with horizon $H$, we encode each action $a_t^{(h)}$ independently via FFAT and append $H$ action tokens to the sequence. The GEN expert then performs world modeling conditioned on the entire action sequence.

\subsection{Interleaved VLAW Modeling via Sequence Plans}
\label{sec:sequenceplan}
A key challenge in unified VLAW modeling is structuring multimodal sequences that support multi-view and multi-step observations together with multi-step control. Following BAGEL~\cite{bagel2024}, we adopt the concept of \textit{sequence plans}, which specify the ordering of different modality tokens within a single autoregressive sequence. A sequence plan defines how visual observations, language goals, actions, and predicted future states are interleaved during training. 

Let $\{V_{t-i}^{(m)}\}$ denote visual observations where $m \in \{1,\dots,M\}$ indexes camera views and $i \in \{0,\dots,K-1\}$ indexes temporal history, allowing the model to condition on multi-view and multi-step observations. Let $L$ denote a language instruction describing the task objective, such as \textit{``pick up the red block and place it into the bowl''}. Let $\{A_t^{(1)}, \dots, A_t^{(H)}\}$ denote an $H$-step action rollout predicted by the policy. A generalized sequence plan $\mathcal{S}$ is defined as
\begin{equation}
\mathcal{S} =
\big[
V_{t-K+1:t}^{1:M},
\;
L,
\;
\hat{A}_t^{1:H},
\;
\hat{V}_{t+1:t+H}^{1:M}
\big].
\label{eq-sequence-plan}
\end{equation}
where tokens marked with $\hat{\cdot}$ denote predicted values with losses applied during training, while the remaining tokens serve as conditioning inputs. In this formulation, the model predicts multi-step actions and future multi-view observations conditioned on the multi-step visual context and language goal. Importantly, sequence plans are flexible and not restricted to a single ordering. Different plans can emphasize different modeling objectives (e.g., predicting actions before future observations or vice versa), and any valid plan can be used during training. In practice, we sample sequence plans from a predefined set to jointly train perception, action prediction, and world modeling within a unified autoregressive framework (see the Appendix~\Cref{app-sequence-plan} for details).

During supervised fine-tuning (SFT), the model receives the full multimodal sequence—including multi-view observations, language instructions, action rollouts, and future frames—and jointly learns action prediction and future frame generation within a single forward pass. For multi-step control, action supervision is applied at each rollout step. At test time, the model conditions on multi-view observations and the language instruction, then generates an $H$-step action sequence autoregressively. The predicted actions are executed in an open-loop manner, while future frame generation is provided for world modeling.

\subsection{LLM-Inspired Multimodal Train-time Data Sampling}
\label{sec:databalance}

Training a unified VLAW model requires combining heterogeneous datasets that differ substantially in scale, modality composition, and supervision signals. 
Naively sampling from the aggregated dataset leads to biased optimization, where large datasets dominate gradient updates while smaller but critical supervision signals—such as action prediction—are underrepresented. 
Such imbalance is widely observed in multimodal foundation model training, where large-scale perception or instruction datasets often exceed robotics trajectory data by orders of magnitude~\cite{alayrac2022flamingo,brohan2023rt2}. We adopt the following sampling strategies together.
\paragraph{Mixture Dataset Sampling.}
Inspired by mixture-of-datasets strategies used in large multimodal models~\cite{alayrac2022flamingo,bagel2024}, we adopt smoothed dataset sampling to balance heterogeneous supervision sources. 
Given multiple training datasets $\{D_i\}$ with sizes $n_i$, each dataset is sampled with probability
\(
p_i \propto n_i^{\alpha},
\)
where $\alpha \in (0,1)$ controls smoothing of dataset imbalance. 
This strategy reduces the dominance of large datasets while preserving proportional coverage of the data distribution.

\paragraph{Priority Sequence-Plan Sampling.}
Beyond dataset-level balancing, VLAW training must also balance supervision across different modeling objectives. 
During training, data is organized hierarchically as \textit{task → demonstration → training pair}. 
We first sample a task according to the dataset mixture distribution from a random categorical distribution, then randomly sample a demonstration trajectory from that task, and finally construct a training pair by selecting a timestep within the trajectory to form the multimodal sequence used as a sequence plan. Similar as~\Cref{eq-sequence-plan}, each sequence plan defines which tokens in the multimodal sequence are treated as prediction targets during training. Furthermore, to balance supervision across policy learning and world modeling, we adopt a priority sampling strategy~\cite{schaul2015prioritized,bagel2024} over sequence plans. 
Let $\mathcal{P}=\{\pi_k\}$ denote the set of candidate sequence plans, where each plan specifies which tokens in the multimodal sequence are treated as prediction targets. 
Each plan is assigned a priority weight $w_k$ that determines its sampling frequency. 
During training, a sequence plan is sampled according to
\(
P(\pi_k)=\frac{w_k}{\sum_j w_j}.
\)
In practice, we distinguish between two plan categories: 
(1) \textit{policy-focused plans}, which supervise multi-step action rollout prediction, and 
(2) \textit{joint VLAW plans}, which supervise both action prediction and future observation prediction. 
Since learning action-conditioned dynamics is critical for embodied control, joint VLAW plans are assigned larger priority weights ($w_{\text{joint}} > w_{\text{policy}}$), making them more likely to be sampled during training. 
This prioritized scheduling ensures that the model frequently observes supervision signals for both control and environment prediction while still maintaining diversity in multimodal training sequences.

\vspace{-4mm}
\section{Experiments}
\label{sec:exp}

\begin{table*}[t]
\begin{center}
\resizebox{\textwidth}{!}{
\begin{tabular}{
x{90pt}|
x{50pt}
x{60pt}|
x{45pt}
x{45pt}
x{45pt}
x{45pt}
x{45pt}}
\toprule
Method &
Multi-task &
World Model &
Spatial &
Object &
Goal &
Long &
Average \\
\midrule
Diffusion Policy~\cite{chi2025diffusion}
& \budui
& \budui
& 78.3
& 92.5
& 68.3
& 50.5
& 72.4 \\
Octo~\cite{team2024octo}
& \dui
& \budui
& 78.9
& 85.7
& 84.6
& 51.1
& 75.1 \\
MDT~\cite{reuss2024multimodal}
& \budui
& \budui
& 78.5
& 87.5
& 73.5
& 64.8
& 76.1 \\
DiT Policy~\cite{hou2024diffusion}
& \dui
& \budui
& 84.2
& 96.3
& 85.4
& 63.8
& 82.4 \\
MaIL~\cite{jia2025mail}
& \budui
& \budui
& 74.3
& 90.1
& 81.8
& 78.6
& 83.5 \\
ThinkAct~\cite{huangthinkact}
& \dui
& \budui
& 88.3
& 91.4
& 87.1
& 70.9
& 84.4 \\
$\pi_{0.5}$~\cite{intelligence2025pi05visionlanguageactionmodelopenworld}
& \dui
& \budui
& 91.2
& 87.5
& 94.3
& 74.1
& 86.8 \\
SmolVLA~\cite{shukor2025smolvla}
& \budui
& \budui
& 93.0
& 94.0
& 91.0
& 77.0
& 88.8 \\
OpenVLA-OFT~\cite{kim2025fine}
& \dui
& \budui
& 97.6
& 98.4
& 97.9
& 94.5
& 97.1 \\
\midrule
RynnVLA-002~\cite{cen2025rynnvla002}
& \budui
& \dui
& 99.0
& 99.8
& 96.4
& 94.4
& 97.4 \\
\rowcolor{cyan!8}
\modelname
& \dui
& \dui
& \textbf{99.2}
& \textbf{99.9}
& \textbf{97.5}
& \textbf{95.3}
& \textbf{98.0} \\
\bottomrule
\end{tabular}}
\end{center}

\caption{\textbf{VLA results on the LIBERO benchmark.}
Success rate (\%) is reported across the four LIBERO task suites: Spatial, Object, Goal, and Long. 
The Multi-task column indicates whether the method is trained jointly across multiple tasks. 
The World Model column indicates whether the method models action-conditioned environment dynamics. 
All methods in the table use continuous action representations.}
\label{tab-libero-vla}
\end{table*}

\subsection{Experimental Setup}

\paragraph{Datasets.}
We evaluate \modelname on three widely used robotic manipulation benchmarks. 
\textbf{LIBERO}~\cite{li2023libero} is a large-scale simulated benchmark designed for long-horizon language-conditioned manipulation, consisting of four task suites (Spatial, Object, Goal, and Long) that together cover 10 complex manipulation scenarios requiring compositional reasoning and multi-step planning. 
\textbf{Language Table}~\cite{lynch2019learning} is a real-world tabletop manipulation dataset where a robot learns language-conditioned behaviors such as pushing blocks to target locations using human demonstrations and natural language instructions. 
\textbf{Franka}~\cite{isaacsim2024} refers to manipulation environments built in NVIDIA IsaacSim using the Franka Emika Panda robot, which include multi-object pick-and-place and rearrangement tasks with physics-based simulation and visual observations. These benchmarks provide complementary evaluation settings spanning simulated long-horizon manipulation, real-world language grounding, and physics-based robot control. For multi-task settings, we train across all tasks using one model checkpoint.

\paragraph{Baselines.} Our primary focus is evaluating models on the comprehensive VLAW task. For this setting, we mainly compare against the closest baseline, RynnVLA-002~\cite{cen2025rynnvla002}, which also aims to unify visual language perception, action prediction, and world modeling within a single framework. We do not include BagelVLA~\cite{hu2026bagelvla} in our comparison because it is a very recent concurrent work and its code and models have not yet been publicly released, preventing reproducible evaluation. In addition, to evaluate performance on the VLA task, we compare against several strong VLA baselines such as OpenVLA-OFT~\cite{kim2025fine} and $\pi_{0.5}$~\cite{intelligence2025pi05visionlanguageactionmodelopenworld}.

\paragraph{Hyperparameters and Other Details.} For the VLA tasks, we use the success rate as a metric. While for the world modeling task, we use metrics of FVD, PSNR, SSIM, LPIPS. We optimize using AdamW with learning rate $2 \times 10^{-5}$, weight decay $0.01$, cosine decay schedule, and global batch size $32$. Models are trained for $80K$ steps. We use $K=32$ in FFAT and FFAD. In our priority sequence-plan sampling algorithm, we set $w_{\text{joint}} = 2$ and $w_{\text{policy}} = 1$. All experiments are performed on 8$\times$H200 GPUs, while results are averaged over five seeds.

\subsection{Study I: Multi-task Performance}
\paragraph{Unified VLAW models outperform task-specific policies.}
We first evaluate the multi-task manipulation performance on the LIBERO~\cite{li2023libero} benchmark. \Cref{tab-libero-vla} compares \modelname against a diverse set of recent policies, including diffusion-based policies, vision–language–action (VLA) models, and unified multimodal systems. Among these methods, only a subset supports joint multi-task training, and even fewer incorporate explicit world modeling capabilities. Our model achieves the best overall performance, reaching an average success rate of \textbf{98.0\%}, outperforming the strongest baseline RynnVLA-002~\cite{cen2025rynnvla002} by +0.6\%. RynnVLA-002~\cite{cen2025rynnvla002} uses separate weights for the VLA policy and world model, and it achieves inferior performance than ours. These results indicate that unified modeling of perception, language, action, and environment dynamics can provide complementary supervision signals that improve control performance.

\paragraph{World modeling improves predictive environment understanding.}
To further analyze the role of world modeling, we evaluate predictive environment modeling quality across three benchmarks. \Cref{tab-worldmodel-benchmark} reports four world modeling metrics~\cite{cen2025rynnvla002} including FVD, PSNR, SSIM, and LPIPS for both action-conditioned and action-free settings, while qualitative results of three benchmarks are provided in~\Cref{fig-vlaw-results}. Across all datasets, \modelname consistently improves prediction quality compared with RynnVLA-002~\cite{cen2025rynnvla002}. For example, on LIBERO, the action-conditioned variant reduces FVD from 389.5 to 373.1 while simultaneously improving PSNR and SSIM and lowering LPIPS.
Similar improvements are observed on Language Table and Franka, suggesting that the learned dynamics model generalizes across a wide range of environments. Furthermore, comparing the two halves of \Cref{tab-worldmodel-benchmark} reveals that action-conditioned prediction consistently outperforms action-free modeling. This result highlights the importance of explicitly incorporating control signals when learning environment dynamics, since future observations depend strongly on executed actions. Overall, these results suggest that unified VLAW modeling adopted in~\modelname improves both control performance and predictive environment understanding.

\begin{table}[t]
\centering
\tablestyle{2.8pt}{1.12}
\begin{tabular}{lcccc|cccc}
\toprule
& \multicolumn{4}{c}{Action-conditioned} & \multicolumn{4}{c}{Action-free} \\
\cmidrule(lr){2-5} \cmidrule(lr){6-9}
Method 
& FVD$\downarrow$ & PSNR$\uparrow$ & SSIM$\uparrow$ & LPIPS$\downarrow$
& FVD$\downarrow$ & PSNR$\uparrow$ & SSIM$\uparrow$ & LPIPS$\downarrow$ \\
\midrule

\multicolumn{9}{c}{LIBERO~\cite{li2023libero}} \\
\midrule
RynnVLA-002~\cite{cen2025rynnvla002}
& 389.5 & 21.74 & 76.92 & 21.02
& 410.3 & 21.18 & 75.84 & 22.61 \\

\rowcolor{cyan!8}
\modelname
& \textbf{373.1} & \textbf{23.88} & \textbf{82.41} & \textbf{16.33}
& 405.4 & 22.32 & 79.15 & 20.28 \\

\midrule
\multicolumn{9}{c}{Language Table~\cite{lynch2019learning}} \\
\midrule
RynnVLA-002~\cite{cen2025rynnvla002}
& 418.7 & 20.94 & 73.65 & 23.51
& 441.6 & 20.37 & 72.41 & 24.63 \\

\rowcolor{cyan!8}
\modelname
& \textbf{392.2} & \textbf{22.61} & \textbf{77.08} & \textbf{19.74}
& 412.8 & 21.42 & 75.13 & 21.88 \\

\midrule
\multicolumn{9}{c}{Franka~\cite{isaacsim2024}} \\
\midrule
RynnVLA-002~\cite{cen2025rynnvla002}
& 462.4 & 21.36 & 74.21 & 22.80
& 489.3 & 20.81 & 72.93 & 24.01 \\

\rowcolor{cyan!8}
\modelname
& \textbf{421.9} & \textbf{23.04} & \textbf{79.56} & \textbf{18.97}
& 447.6 & 22.02 & 76.88 & 21.35 \\

\bottomrule
\end{tabular}

\vspace{2mm}
\caption{
\textbf{World modeling comparisons.}
We compare action-conditioned and action-free world modeling.
Numbers are averaged across tasks within each benchmark. Our model consistently improves prediction fidelity, especially when action-conditioned.
}
\label{tab-worldmodel-benchmark}
\end{table}
\begin{figure}[t]
    \centering
    \includegraphics[width=\linewidth]{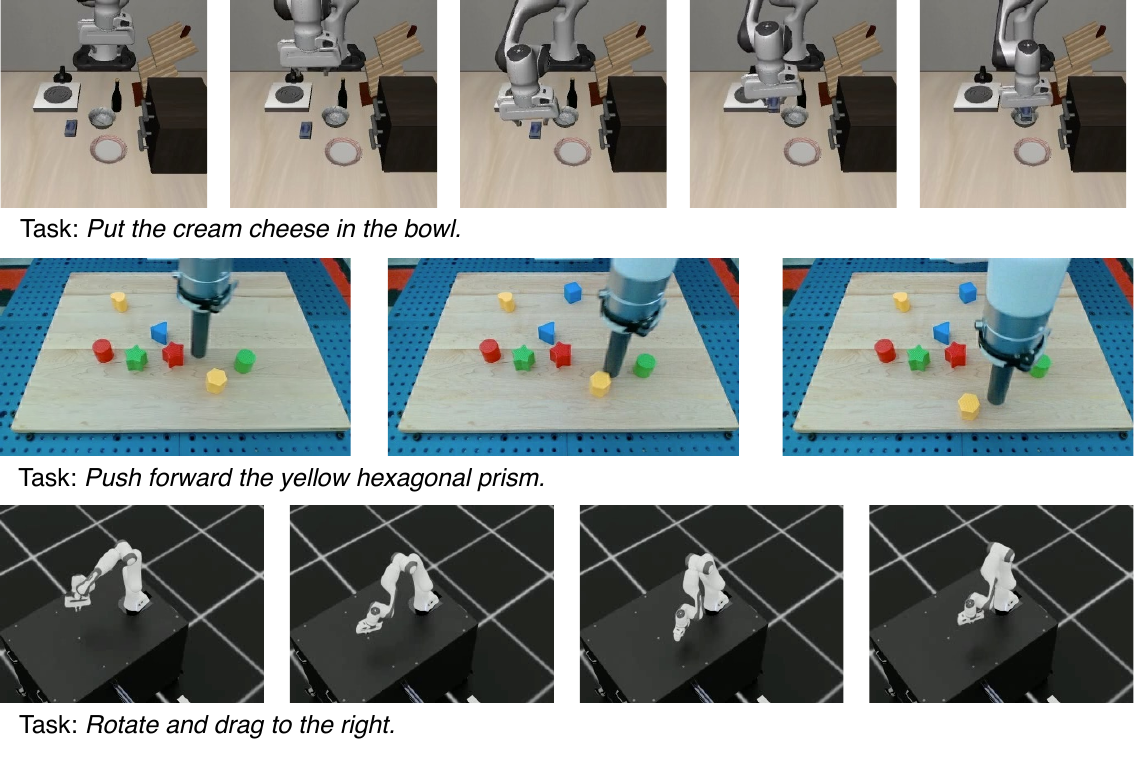}
    \caption{
    World modeling results on three environments: LIBERO~\cite{li2023libero}, Language Table~\cite{lynch2019learning} and Franka~\cite{isaacsim2024}. Please refer to the supplementary to see more video demos.
    }
    \label{fig-vlaw-results}
\vspace{-3mm}
\end{figure}



\subsection{Study II: Action Representation Quality}

\paragraph{Action decoder design.}
\Cref{subtab-action-decoder} compares different action decoding strategies. 
For \textbf{regression}, the policy directly predicts continuous actions using a linear projection head trained with an $\ell_2$ loss. 
For \textbf{bin discretization}, each action dimension is uniformly quantized into 256 bins, and the model predicts the bin index using a categorical cross-entropy loss, with the final action obtained by mapping the predicted bin center back to the continuous space. 
For \textbf{FAST}~\cite{pertsch2025fast}, actions are tokenized using the FAST tokenizer with a codebook size of 1024 and sequence length of 4 tokens per action step. Our FFAD structure is detailed in~\Cref{sec:ffad}, which reduces the action MSE (A-MSE) to \textbf{0.028} and improves the success rate to \textbf{98.0\%}, outperforming other action prediction methods.

\paragraph{Number of Fourier Bands in FFAD / FFAT.}
\Cref{subtab:fourier} studies the effect of the number of Fourier bands $K$ used in FFAD. 
Increasing $K$ improves both action prediction and success rate up to $K=32$, while further increasing to $K=64$ yields no additional gains. 
This indicates that $K=32$ provides a good trade-off between representation capacity and stability. We further examine the number of Fourier bands used in the Fourier Feature Action Tokenizer (FFAT). As shown in \Cref{subtab-ffat}, PSNR increases from 23.21 to \textbf{23.88} when $K$ increases from 8 to 32, with marginal changes beyond that point.

\paragraph{Representation Structure Analysis.}
We analyze the structure of the learned action representations across different decoding strategies, including regression, bin discretization, FAST~\cite{pertsch2025fast}, and our proposed FFAD. 
For each method, we extract the action embeddings from the trained policy by taking the decoder input representation corresponding to each action step. 
Using these frozen embeddings, we train a linear classifier (linear probe) to predict the task identity among the LIBERO task suites using only the action representation as input. 
The classifier is a single linear layer trained with cross-entropy loss, while the backbone model remains fixed. Higher accuracy indicates that the representation encodes more structured task-relevant information. 
As shown in \Cref{subtab-structure}, the linear probe accuracy gradually improves from 49.8\% (regression) to 52.1\% (bin) and 54.2\% (FAST), while FFAD achieves a significantly higher accuracy of \textbf{68.7\%}. 
This result suggests that our proposed FFAD organizes the action embedding space more effectively, producing features that better capture task-level structure.

\begin{table}[t!]
\caption{\textbf{Ablation studies.}
We analyze several design choices of our~\modelname model on the LIBERO validation set. The default model is highlighted in \colorbox{cyan!8}{Cyan}.}
\label{tab-ablation}
\vspace{-5mm}
\begin{minipage}{0.48\textwidth}
\centering
\subcaption{\textbf{Action Prediction Methods}}
\label{subtab-action-decoder}
\vspace{-3mm}
\tablestyle{0pt}{1.08}
\begin{tabular}{x{39pt}| x{38pt} x{30pt} x{30pt} x{30pt}}
\toprule
Method & Regression & Bin & FAST & \cellcolor{cyan!8} FFAD \\
\midrule
A-MSE $\downarrow$ & 0.042 & 0.037 & 0.035 & \cellcolor{cyan!8}\textbf{0.028} \\
Success $\uparrow$ & 96.3 & 95.6 & 96.9 & \cellcolor{cyan!8}\textbf{98.0} \\
\bottomrule
\end{tabular}
\end{minipage}
\hspace{1.5mm}
\begin{minipage}{0.49\textwidth}
\centering
\subcaption{\textbf{Number of Fourier Bands in FFAD}}
\label{subtab:fourier}
\vspace{-3mm}
\tablestyle{0pt}{1.08}
\begin{tabular}{x{39pt}| x{38pt} x{30pt} x{30pt} x{30pt}}
\toprule
$K$ & 8 & 16 & \cellcolor{cyan!8} 32 & 64 \\
\midrule
A-MSE $\downarrow$ & 0.034 & 0.031 & \cellcolor{cyan!8}\textbf{0.028} & 0.029 \\
Success $\uparrow$ & 96.8 & 97.3 & \cellcolor{cyan!8}\textbf{98.0} & 97.6 \\
\bottomrule
\end{tabular}
\end{minipage}
\begin{minipage}{0.48\textwidth}
\centering
\subcaption{\textbf{Number of Fourier Bands in FFAT}}
\label{subtab-ffat}
\vspace{-3mm}
\tablestyle{0pt}{1.08}
\begin{tabular}{x{45pt}| x{30pt} x{30pt} x{30pt} x{30pt}}
\toprule
$K$ & 8 & 16 & \cellcolor{cyan!8} 32 & 64 \\
\midrule
PSNR $\uparrow$ & 23.21 & 23.82 & \cellcolor{cyan!8}\textbf{23.88} & 23.87 \\
\bottomrule
\end{tabular}
\end{minipage}
\hspace{1.5mm}
\begin{minipage}{0.48\textwidth}
\centering
\subcaption{\textbf{Representation Linear Probe}}
\label{subtab-structure}
\vspace{-3mm}
\tablestyle{0pt}{1.08}
\begin{tabular}{x{34pt}| x{40pt} x{28pt} x{28pt} x{32pt}}
\toprule
Metric & Regression & Bin & FAST & \cellcolor{cyan!8} FFAD \\
\midrule
Acc. $\uparrow$ & 49.8 & 52.1 & 54.2 & \cellcolor{cyan!8}\textbf{68.7} \\
\bottomrule
\end{tabular}
\end{minipage}
\begin{minipage}{0.48\textwidth}
\centering
\subcaption{\textbf{Stability Test - Action Noise}}
\label{subtab-noise}
\vspace{-3mm}
\tablestyle{0pt}{1.08}
\begin{tabular}{x{38pt}| x{70pt} x{65pt}}
\toprule
Metric & RynnVLA-002~\cite{cen2025rynnvla002} & \cellcolor{cyan!8} \modelname \\
\midrule
PSNR $\uparrow$ & 21.38 & \cellcolor{cyan!8}\textbf{23.64} \\
LPIPS $\downarrow$ & 0.274 & \cellcolor{cyan!8}\textbf{0.203} \\
\bottomrule
\end{tabular}
\end{minipage}
\hspace{1.5mm}
\begin{minipage}{0.48\textwidth}
\centering
\subcaption{\textbf{Stability Test - Action Scaling}}
\label{subtab-scale}
\vspace{-3mm}
\tablestyle{0pt}{1.08}
\begin{tabular}{x{38pt}| x{70pt} x{60pt}}
\toprule
Metric & RynnVLA-002~\cite{cen2025rynnvla002} & \cellcolor{cyan!8} \modelname \\
\midrule
PSNR $\uparrow$ & 21.91 & \cellcolor{cyan!8}\textbf{23.57} \\
LPIPS $\downarrow$ & 0.261 & \cellcolor{cyan!8}\textbf{0.208} \\
\bottomrule
\end{tabular}
\end{minipage}
\begin{minipage}{0.48\textwidth}
\centering
\subcaption{\textbf{Stability Test - Temporal Perturbation}}
\label{subtab-temporal}
\vspace{-3mm}
\tablestyle{0pt}{1.08}
\begin{tabular}{x{38pt}| x{70pt} x{60pt}}
\toprule
Metric & RynnVLA-002~\cite{cen2025rynnvla002} & \cellcolor{cyan!8} \modelname\\
\midrule
PSNR $\uparrow$ & 20.94 & \cellcolor{cyan!8}\textbf{23.41} \\
LPIPS $\downarrow$ & 0.287 & \cellcolor{cyan!8}\textbf{0.214} \\
\bottomrule
\end{tabular}
\end{minipage}
\hspace{1.5mm}
\begin{minipage}{0.48\textwidth}
\centering
\subcaption{\textbf{Eigenvalue Spectrum}}
\label{subtab-spectrum}
\vspace{-3mm}
\tablestyle{0pt}{1.08}
\begin{tabular}{x{70pt}| x{70pt} x{30pt}}
\toprule
Metric & RynnVLA-002~\cite{cen2025rynnvla002} & \cellcolor{cyan!8} Ours \\
\midrule
Effective Rank $\uparrow$ & 16.3 & \cellcolor{cyan!8}\textbf{33.7} \\
$\lambda_1 / \sum_i \lambda_i \;\downarrow$ & 0.67 & \cellcolor{cyan!8}\textbf{0.43} \\
\bottomrule
\end{tabular}
\end{minipage}
\vspace{-4mm}
\end{table}
\vspace{-3mm}
\subsection{Study III: Stability Under Distribution Shifts}

Robotic manipulation systems frequently encounter distribution shifts in action dynamics caused by execution noise, trajectory scaling, or temporal inconsistencies. A robust world model should therefore capture the underlying action-conditioned transition dynamics rather than simply reproducing previously observed trajectories. We evaluate stability on the LIBERO benchmark by introducing three controlled perturbations to the action sequence \textit{during evaluation}: Gaussian action noise with variance $\sigma=0.05$, multiplicative action scaling sampled from $[0.8,1.2]$, and temporal perturbations that shift actions by $\pm1$ timestep. As shown in \Cref{subtab-noise,subtab-scale,subtab-temporal}, RynnVLA-002 experiences substantial degradation in prediction quality under action perturbations. 
For example, under action noise the PSNR drops to 21.38 with LPIPS increasing to 0.274, while our model maintains significantly higher prediction fidelity with a PSNR of \textbf{23.64} and LPIPS of \textbf{0.203}. 
Similar trends are observed for action scaling and temporal perturbations, where our method consistently achieves higher PSNR and lower LPIPS than RynnVLA-002. 
These results indicate that RynnVLA-002 primarily relies on visual history and tends to reproduce previously observed manipulation trajectories, whereas our action-conditioned model better captures the underlying action dynamics. To further analyze representation stability, we compute the eigenvalue spectrum of the action embedding covariance matrix. 
As shown in \Cref{subtab-spectrum}, our model achieves a significantly higher effective rank (33.7 vs.\ 16.3) and a lower dominant eigenvalue ratio ($\lambda_1/\sum_i\lambda_i$ of 0.43 vs.\ 0.67) compared with RynnVLA-002. 
This indicates that our representation distributes variance across more directions in the embedding space, resulting in a richer and more stable representation of action dynamics under distribution shifts.
\vspace{-5mm}
\section{Conclusion}
\vspace{-3mm}
We introduced \modelname, a unified Vision–Language–Action–World (VLAW) model built upon BAGEL’s two-tower multimodal architecture. By jointly modeling perception, language grounding, action prediction, and future observation generation within a single autoregressive framework, our approach enables unified multimodal world modeling for embodied agents~\cite{yuan2026clarify, zhao2026rma}. Experiments across robotic manipulation benchmarks show that unified VLAW modeling improves multi-task control performance, produces more structured and transferable action representations through Fourier-based tokenization, and yields more stable representations under distribution shifts. These results suggest that architectural unification can improve both policy learning and predictive environment modeling, providing a promising foundation for scalable multimodal world models.

\clearpage  


%
%
\bibliographystyle{splncs04}
\bibliography{main}

\clearpage
\newpage
\appendix
\section*{Appendix}
\addcontentsline{toc}{section}{Appendix}
\section{FFAT and FFAD Mathematical Properties}
\label{app-math-behind-fatt-ffad}

In this section we provide theoretical analysis for the Fourier Feature Action Tokenizer (FFAT, introduced in~\Cref{sec:ffat}) and the Fourier Feature Action Decoder (FFAD, introduced in~\Cref{sec:ffad}). 
We analyze three key properties related to our proposed models:

\begin{enumerate}
\item stability of the Fourier action representation,
\item information preservation of the Fourier embedding,
\item consistency of the Fourier-based reconstruction used in FFAD.
\end{enumerate}

Throughout this section we assume actions are bounded in a compact interval
$
a \in \mathcal{A} \subset [-1,1]^d
$
which is standard for robot control signals.

\subsection{Stability of the Fourier Action Representation}

We first show that the Fourier embedding used by FFAT is Lipschitz continuous, ensuring robustness to small perturbations of actions.

\paragraph{Definition.}
Let the Fourier feature mapping used in FFAT be (introduced in~\Cref{eq-fourier-feature} of the main paper):

\begin{equation}
\phi(a)=
\left[
\sin(2^0\pi a),\cos(2^0\pi a),
\dots,
\sin(2^K\pi a),\cos(2^K\pi a)
\right].
\end{equation}

\begin{theorem}[Lipschitz Stability of Fourier Action Embedding]
\label{thm:lipschitz}
For any $a_1,a_2\in\mathcal{A}$, the Fourier embedding $\phi(\cdot)$ satisfies

\begin{equation}
\|\phi(a_1)-\phi(a_2)\|_2
\le
L\|a_1-a_2\|_2,
\end{equation}
where
\begin{equation}
L=\pi\sqrt{2\sum_{k=0}^{K}2^{2k}} .
\end{equation}

\end{theorem}
Theorem~\ref{thm:lipschitz} shows that the Fourier representation is stable: small perturbations in actions produce bounded changes in token representations.
\begin{proof}

For each frequency $k$, consider the mapping:

\begin{equation}
f_k(a)=(\sin(2^k\pi a),\cos(2^k\pi a)).
\end{equation}

The Jacobian norm of $f_k$ satisfies:

\begin{equation}
\|J_{f_k}(a)\|_2 \le 2^k\pi.
\end{equation}

Therefore,

\[
\|f_k(a_1)-f_k(a_2)\|_2
\le
2^k\pi |a_1-a_2|.
\]

Stacking all frequency components,

\begin{equation}
\|\phi(a_1)-\phi(a_2)\|_2^2
\le
\sum_{k=0}^{K}
2(2^k\pi)^2
\|a_1-a_2\|_2^2 .
\end{equation}

Taking square roots completes the proof.
\end{proof}

\subsection{Information Preservation of Fourier Features}

Next we show that the Fourier embedding retains sufficient information about the action.

\begin{theorem}[Injectivity on Bounded Domains]
\label{thm:injective}

Assume $a \in (-1,1)$ and $K\ge1$. 
Then the Fourier embedding $\phi(a)$ is injective almost everywhere on $\mathcal{A}$.

\end{theorem}

\begin{proof}

Suppose

\[
\phi(a_1)=\phi(a_2).
\]

Then for every frequency $k$

\[
\sin(2^k\pi a_1)=\sin(2^k\pi a_2),
\quad
\cos(2^k\pi a_1)=\cos(2^k\pi a_2).
\]

This implies

\[
2^k\pi a_1 \equiv 2^k\pi a_2 \pmod{2\pi}.
\]

Thus

\[
a_1-a_2=\frac{m}{2^k}
\]

for some integer $m$.  
For $k=0$ and $k=1$ simultaneously this yields

\[
a_1-a_2=m_0 = \frac{m_1}{2}.
\]

The only solution in $(-1,1)$ is $a_1=a_2$.  
Therefore the mapping is injective except at measure-zero boundary points.
\end{proof}

This result ensures that Fourier tokens preserve the identity of actions within the bounded action domain.

\subsection{Consistency of Fourier Reconstruction}

We now analyze the phase-based reconstruction used by FFAD.

\begin{theorem}[Consistency of Phase Reconstruction]
\label{thm:reconstruction}

Let $\hat{\phi}(a)$ be the predicted Fourier features satisfying

\begin{equation}
\|\hat{\phi}(a)-\phi(a)\|_2 \le \epsilon .
\end{equation}

Then the reconstructed action $\hat{a}$ produced by the FFAD decoder satisfies

\begin{equation}
|\hat{a}-a|
\le
C\epsilon
\end{equation}

for some constant $C$ depending only on $K$.

\end{theorem}

\begin{proof}

Let

\begin{equation}
\hat{s}_k=\sin(2^k\pi a)+\delta_k^s,
\quad
\hat{c}_k=\cos(2^k\pi a)+\delta_k^c
\end{equation}

with

\begin{equation}
(\delta_k^s)^2+(\delta_k^c)^2 \le \epsilon^2 .
\end{equation}

The phase estimate satisfies

\begin{equation}
\hat{\theta}_k = \operatorname{atan2}(\hat{s}_k,\hat{c}_k).
\end{equation}

Using standard perturbation bounds for $\operatorname{atan2}$, we obtain

\begin{equation}
|\hat{\theta}_k-\theta_k|
\le
C_1 \epsilon .
\end{equation}

Dividing by $2^k\pi$ yields

\begin{equation}
|\hat{a}_k-a|
\le
\frac{C_1}{2^k\pi}\epsilon .
\end{equation}

Averaging across frequencies produces

\begin{equation}
|\hat{a}-a|
\le
C\epsilon
\end{equation}

where

\begin{equation}
C=\frac{C_1}{K+1}
\sum_{k=0}^{K}\frac{1}{2^k\pi}.
\end{equation}

\end{proof}

This theorem shows that FFAD reconstruction is stable: errors in predicted Fourier features translate into bounded errors in reconstructed actions.

\subsection{Approximation Advantage of Fourier Action Modeling}

We now relate the Fourier representation used in FFAT/FFAD to approximation theory~\cite{FourierFeaturesLearn}. 
Intuitively, representing actions using sinusoidal basis functions enables linear models to approximate nonlinear functions of the action variable.

\begin{theorem}[Approximation with Fourier Features]
\label{thm:approximation}

Let $\mathcal{A}\subset[-1,1]$ be a compact domain and let 
$f:\mathcal{A}\rightarrow\mathbb{R}$ be a continuous function. 
Define the $K$-band Fourier feature embedding

\begin{equation}
\phi_K(a)=
\left[
\sin(\pi a),\cos(\pi a),
\dots,
\sin(2^K\pi a),\cos(2^K\pi a)
\right].
\end{equation}

Then for any $\epsilon>0$ there exists an integer $K$ and a vector of weights $w$ such that

\begin{equation}
\sup_{a\in\mathcal{A}} 
\big| f(a)-w^\top \phi_K(a) \big|
<\epsilon .
\end{equation}

\end{theorem}

\begin{proof}

By the Stone--Weierstrass theorem, trigonometric polynomials are dense in the space of continuous functions on a compact interval. 
Therefore for any $\epsilon>0$ there exists a trigonometric polynomial

\begin{equation}
p_K(a)=
\sum_{k=-K}^{K} c_k e^{ik\pi a},
\end{equation}
such that
\begin{equation}
\sup_{a\in\mathcal{A}}
|f(a)-p_K(a)|
<\epsilon .
\end{equation}

Using Euler's identity,
\begin{equation}
e^{ik\pi a}
=
\cos(k\pi a)+i\sin(k\pi a),
\end{equation}
the polynomial $p_K(a)$ can be rewritten as a linear combination of sine and cosine functions. 
Since the Fourier embedding $\phi(a)$ explicitly contains these sinusoidal components across multiple frequencies, $p_K(a)$ can be written in the form

\begin{equation}
p_K(a)=w^\top \phi(a),
\end{equation}
for some coefficient vector $w$. 
Thus a linear model operating on Fourier features can approximate $f(a)$ with arbitrary accuracy as $K$ increases.

\end{proof}

Theorem~\ref{thm:approximation} shows that Fourier feature embeddings allow linear predictors to represent a broad class of nonlinear mappings from actions to target variables. 
This property helps explain why predicting actions in Fourier space can simplify the learning of complex action–state relationships.

\section{Sequence Plan Details and Study}
\label{app-sequence-plan}

In this section we provide additional details on the sequence-plan design used in \modelname, where sequence plans are firstly introduced in~\Cref{sec:sequenceplan}.

We implement several sequence-plan templates during training. 
Let $V_t$ denote visual tokens at time $t$, $L$ the language instruction, $A_t$ the action at time $t$, and $\hat{V}_{t+1}$ the predicted future observation. In practice, we use policy-focused plans and joint VLAW plans. The sampling strategy is mentioned in~\Cref{sec:databalance}. 

\paragraph{Plan A: Policy-focused plan.}

This plan prioritizes action prediction before world prediction:

\begin{equation}
[V_{t-K+1:t}, L, \hat{A}_t^{1:H}, \hat{V}_{t+1:t+H}]
\end{equation}

Here the model first predicts the action rollout and then generates future observations conditioned on those predicted actions.

\paragraph{Plan C: Interleaved plan.}

This plan interleaves action and world predictions:

\begin{equation}
[V_{t-K+1:t}, L, \hat{A}_t^{1}, \hat{V}_{t+1}, \dots, \hat{A}_t^{H}, \hat{V}_{t+H}]
\end{equation}

This sequence encourages tighter coupling between control signals and environment dynamics.
\end{document}